\documentclass{article} 
\usepackage{hyperref}
\usepackage[accepted]{arxiv_version}

\usepackage{amsmath,amsfonts,bm}

\def\eqref#1{equation~\ref{#1}}

\def\1{\bm{1}}

\DeclareMathAlphabet{\mathsfit}{\encodingdefault}{\sfdefault}{m}{sl}
\SetMathAlphabet{\mathsfit}{bold}{\encodingdefault}{\sfdefault}{bx}{n}

\usepackage{times}
\usepackage{amsmath}
\usepackage{amssymb}
\usepackage{graphicx}
\usepackage{caption}
\usepackage{subcaption}
\usepackage{xcolor}
\usepackage{booktabs}
\usepackage[normalem]{ulem}

\usepackage{amsthm}
\usepackage{thmtools} 
\usepackage{thm-restate}
\newtheorem{theorem}{Theorem}
\newtheorem{remark}{Remark}
\newtheorem{definition}{Definition}
\newtheorem{lemma}{Lemma}

\newcommand{\Lip}{\operatorname{Lip}}
\newcommand{\pad}{\operatorname{pad}}

\newcommand{\diag}{\operatorname{diag}}

\newcommand{\arctanh}{\operatorname{arctanh}}
\newcommand{\DFT}{\mathcal{F}}
\newcommand{\ConvOp}{\mathcal{K}}
\newcommand{\MatrixConvOp}{{K}}
\newcommand{\X}{\mathcal{X}}

\newcommand{\toVec}{\operatorname*{Vec}}

\begin{document}
\twocolumn[
\icmltitle{Depthwise Separable Convolutions Allow for Fast and Memory-Efficient Spectral Normalization}



\icmlsetsymbol{equal}{*}

\begin{icmlauthorlist}
\icmlauthor{Christina Runkel}{equal,si}
\icmlauthor{Christian Etmann}{equal,cam}
\icmlauthor{Michael M\"{o}ller}{si}
\icmlauthor{Carola-Bibiane Sch\"{o}nlieb}{cam}
\end{icmlauthorlist}

\icmlaffiliation{si}{Department for Computer Science and Electrical Engineering, University of Siegen, Siegen, Germany}
\icmlaffiliation{cam}{Department of Applied Mathematics and Theoretical Physics, University of Cambridge, Cambridge, UK}

\icmlcorrespondingauthor{Christina Runkel}{christina.runkel@student.uni-siegen.de}
\icmlcorrespondingauthor{Christian Etmann}{cetmann@damtp.cam.ac.uk}

\icmlkeywords{}

\vskip 0.3in
]



\printAffiliationsAndNotice{\icmlEqualContribution} 

\begin{abstract}
An increasing number of models require the control of the spectral norm of convolutional layers of a neural network. While there is an abundance of methods for estimating and enforcing upper bounds on those during training, they are typically costly in either memory or time. In this work, we introduce a very simple method for spectral normalization of depthwise separable convolutions, which introduces negligible computational and memory overhead. We demonstrate the effectiveness of our method on image classification tasks using standard architectures like MobileNetV2.
\end{abstract}
\section{Introduction} \label{sec:introduction}

Neural networks, as universal approximators (of e.g. continuous functions), may represent highly complex functions, which in practice results in them being quite sensitive to their inputs, as evidenced e.g. by adversarial examples \citep{szegedy2013intriguing}. One simple measure of how sensitive to its inputs a network is, is the network's Lipschitz constant. Indeed, an increasing number of applications call for the control of this Lipschitz constant, which yields guarantees for the input-output-sensitivity. One particular way of controlling this Lipschitz constant of a network is by \emph{spectral normalization} of its convolutions and weight matrices, which explicitly limits the network's Lipschitz constant in terms of the 2-norm.

Since it determines a trade-off between expressiveness and robustness, controlling a network's Lipschitz constant is a form of regularization \citep{oberman2018lipschitz}. A direct application of this principle is the limiting of the network's Lipschitz constant in order to robustify it against adversarial attacks \citep{NEURIPS2018_48584348}.
A version of this idea is also found in spectrally normalized GANs \citep{miyato2018spectral}, which regularize the discriminator and which results in better-looking images. Note that in practice this is similar to Wasserstein GANs \citep{arjovsky2017wasserstein}, which require enforcing the Lipschitz constant of their critic network to be lower or equal to one (in order to approximate a Wasserstein distance). Another instance in which spectral normalization can be used to enforce necessary guarantees is found in invertible residual networks \citep{behrmann2019invertible}, which require the Lipschitz constant of the residual blocks to be below one. Furthermore, the Lipschitz constant of an invertible network influences the stability of inversion, as shown in \citep{behrmann2020}. For example in additive coupling layers \citep{dinh2014nice}, limiting the layers' Lipschitz constant will result in an increase in \emph{both} the forward and inverse numerical stability. This in particular has implications for training with restored (instead of stored) activations \citep{gomez2017reversible}.

To calculate and (thus enforce) spectral norms on convolutional layers, common methods either use the power method to calculate the norms, or use (typically non-tight) upper bounds as an approximation. Enforcing these spectral norms is usually done by performing a normalization step (i.e. dividing the layers' weights by the computed Lipschitz constant) after each step of (stochastic) gradient descent, which means that the calculation of the spectral norms has to be performed many times during training. This means that the spectral normalization can strongly influence the overall speed of training, potentially bottlenecking the training duration. Details on this are given in section \ref{subsec:spectral_normalization}.

In this work, we propose a novel method for spectral normalization tailored to depthwise separable convolutions \citep{sifre2014rigid}. Depthwise separable convolutions are a form of convolutional layers, in which the channel-wise convolution and the inter-channel combination of channels are decoupled.
Our proposed method makes use of the specific structure of depthwise separable convolutions to achieve significant speedups compared to other methods, while having practically no additional memory overhead. In a realistic training scenario (training MobileNetV2 on ImageNet), each epoch takes only 2\% additional training time -- thus making spectral normalization essentially 'for free' for depthwise separable convolutional networks.

\section{Problem setting}\label{sec:problem_setting}
\subsection{Depthwise separable convolutions}\label{subsec:dw_conv}
Let $\theta^{\langle i, j \rangle}$ denote a filter corresponding to an input's $j$-th channel (denoted $x^{\langle j \rangle}$) and an output's $i$-th channel (denoted $y^{\langle i \rangle}$). Then
\begin{equation}\label{eq:convolutions}
    y^{\langle i \rangle} = \textstyle{\sum_j} \theta^{\langle i, j \rangle} \ast x^{\langle j \rangle}
\end{equation}
describes the conventional multi-channel convolution operation, where $\ast$ signifies a spatial filtering (a convolution, respectively a cross-correlation). The spatial and inter-channel computations of these multi-channel convolutions are hence inherently linked. Recently, \emph{depthwise separable convolutions} have arisen as a family of computationally cheaper (and yet expressive) alternatives, in which the spatial and inter-channel computations are decoupled in the form of \emph{depthwise} respectively \emph{pointwise} convolutions. Depthwise convolutions apply filtering to each input channel independently (resulting in the same number of output channels), which can be equivalently represented by setting $\theta^{\langle i, j \rangle}=0$ for all $i \neq j$ in equation (\ref{eq:convolutions}). These serve as the spatial operations of depthwise separable convolutions. Pointwise convolutions on the other hand can be represented by multi-channel convolutions as in equation (\ref{eq:convolutions}) for which $\theta^{\langle i, j \rangle} \in \mathbb{R}$ (e.g. $1 \times 1$ convolutions in 2D), i.e. for non-strided filterings the output channels are merely linear combinations of the input channels. They constitute the inter-channel computations of depthwise separable convolutions. 

The various realizations of depthwise separable convolutions in the literature mainly differ by their number and order of operations and where (and how often) a nonlinearity is applied. While a depthwise convolution followed by a pointwise convolution is a low-rank factorization of a conventional convolution \citep{sifre2014rigid}, recent approaches usually incorporate nonlinearities between the two types of operations \citep{xception, howard2017, sandler2018mobilenetv2, efficientnet}. It is worth highlighting that EfficientNet \citep{efficientnet}, which employs the depthwise separable convolutions defined by MobileNetV2 \citep{sandler2018mobilenetv2}, is currently state-of-the-art in the ILSVRC (ImageNet) challenge \citep{pham2020meta}.

\subsection{Mathematical preliminaries}\label{subsec:math_prelim}
A function $g: \X \rightarrow \mathcal{Y}$ is called Lipschitz-continuous on two normed spaces $(\X,\|\cdot\|_{\X})$ and $({\mathcal{Y}},\|\cdot\|_\mathcal{Y})$, if there is a constant $L \geq 0$ such that $$\|g(x_1)-g(x_2)\|_{\mathcal{Y}} \leq L\|x_1-x_2\|_\X$$ for all $x_1, x_2 \in \X$. The smallest such constant $L$ is referred to as the Lipschitz-constant of $g$ and will be denoted $\Lip (g)$. \newline

For continuous, almost everywhere differentiable functions, the Lipschitz constant of a function $g$ is given by the supremum of the operator norm (w.r.t. the norms of $\X$, ${\mathcal{Y}}$) of all possible derivatives, i.e. $\Lip (g) = \sup_{x \in \X} \| \partial g(x)  \|_{\X \to {\mathcal{Y}}}$. In the following, we will assume that $\X$ and ${\mathcal{Y}}$ are Euclidean spaces (equipped with the 2-norm), in which case the above operator norm becomes the \emph{spectral norm} of $\partial g(x)$ (where defined), i.e. the largest singular value of $\partial g(x)$. This highlights the difficulty of actually calculating the Lipschitz constant of a neural network -- even in the case of locally linear functions, such as e.g. ReLU networks, this would require the evaluation of the operator norm in each locally linear region. In fact, even for a two layer neural network, calculating the Lipschitz constant is NP-hard \citep{scaman2018}. Approximating this Lipschitz constant can be achieved by phrasing the Lipschitz estimation problem as an optimization problem \citep{fazlyab2019, latorre2020, chen2020semialgebraic}. These are, however, still very costly.

On the other hand, calculating \emph{upper bounds} of Lipschitz constants of neural networks is tractable, if finding an upper bound of the individual layers' Lipschitz constants is tractable. In particular, if $g=f_n \circ \cdots \circ f_1$ for Lipschitz continuous functions $f_i$, then $\Lip (g) \leq \prod_i^n \Lip (f_i)$  (\emph{chain rule for Lipschitz constants}). While the bounds this chain rule yields may be too loose for approximating the true Lipschitz constant of the \emph{whole network}, it may still be possible to make approximate statements about \emph{individual layers'} Lipschitz constants.
For layers of the type $f(x)=\sigma (Ax + b)$, where $A$ is linear, $b$ is a bias vector and $\sigma$ is an activation function with bounded Lipschitz constant, it holds that
\begin{equation}
    \begin{aligned} \label{eq:lipschitz_per_layer}
        \Lip (f) &\leq \Lip (\sigma) \cdot \Lip (A)  = \Lip (\sigma) \|A\|_2,
    \end{aligned}
\end{equation}
i.e.~the Lipschitz constant of a nonlinear layer can typically be upper bounded by the spectral norm of its linear portion. This is true in particular for convolutional layers or dense layers with e.g.~ReLU, tanh, sigmoid, softmax or ELU nonlinearity (all of which have $\Lip (\sigma)=1)$.

\subsection{Spectral normalization}\label{subsec:spectral_normalization}
One particularly important insight about Lipschitz constants is the fact that the ability to calculate them often enables one to enforce a specific Lipschitz constant in a function. 

In fact, it holds that $\Lip (\lambda f)=\lambda \Lip (f)$ for any $\lambda >0$, so that $\tilde{f}:= \lambda \cdot f/ \Lip(f)$ has $\Lip(\tilde{f})= \lambda$ for some desired $\lambda>0$. Instead of scaling the function itself, most types of neural network layers admit a scaling of their associated parameters: Since for example convolutions are positively 1-homogeneous in their kernels, it suffices to scale their filter kernel accordingly. That means that for a convolutional operator $\ConvOp_\theta$ with filter kernel $\theta$ (and for our setting of spaces equipped with the 2-norm), it holds that by rescaling this kernel via $\tilde{\theta}=\theta / \| \ConvOp_\theta \|_2$ (and keeping a possible bias), one obtains a new convolutional operator $\ConvOp_{\tilde{\theta}}$ with $\Lip(\ConvOp_{\tilde{\theta}})=1$. In the following, we will refer to this normalization procedure as \textbf{spectral normalization}. A common way of training neural networks with spectrally normalized convolutional kernels is to apply spectral normalization after each optimization step. Note that if we overestimate the spectral norm, then still $\Lip(\ConvOp_{\tilde{\theta}})\leq 1$.

If one wishes to enforce some specific Lipschitz constant, multiplying the output of the normalized convolution by some desired $K>0$ then makes the overall mapping $K$-Lipschitz. We call this method \textbf{hard scaling} and refer $K$ to the \emph{scaling constant}. \newline
One can make this scaling a bit more flexible by instead multiplying with $K\cdot \tanh (s)$, where $s\in \mathbb{R}$ is a learnable parameter. Since $\tanh (s) \in (-1,1)$, this will bound the enforced Lipschitz constant to be below $K$. Due to the dependence on the learnable parameter $s$, one allows the 'right' Lipschitz constant up to $K$ to be learned. We will refer to this variant as \textbf{soft scaling}. 

\subsection{Spectral norm estimation for convolutional layers} \label{subsec:specnorm_for_conv_layers}
Unlike in the case of dense layers, whose parameters directly allow for a computation of their spectral norm via matrix methods, it is more difficult to assess a convolution's norm solely from its parameters. In fact, the same kernel generally defines not one, but a whole family of operators (depending e.g. on the input resolution and stride). The estimation of their norms can either be performed with generic methods or make use of the operators' specific structure.
One approach is to restrict the convolutional operators to be exactly orthogonal \citep{xiao2018dynamical,li2019preventing} to strictly enforce unit-norm convolutions. By relaxing this constraint to approximately orthogonal convolutions, the spectral norm is likewise relaxed to values close to \citep{cisse2017} or up to \citep{qian2018} one.
Another class of methods is to exploit the structure of multi-channel convolution operators to gain insights on their singular values. In particular, by making use of the theory of doubly block circulant matrices \citet{sedghi2019} are able to derive an exact formula for the singular values of \emph{circulant} 2D convolutions, using Fourier transforms and a singular value decomposition. While this is exact, the values are only approximate for the more commonly used zero-padded or 'valid' convolutions. Furthermore, the computation is quite slow. Similarly-derived and faster to compute upper bounds can be derived for circulant convolutions \citep{singla2019} or padded convolutions \citep{araujo2020}. As these methods are designed for 2D convolutions (using the structure of doubly block circulant/Toeplitz matrices), it is furthermore not immediately clear how to extend these to 3D.

A very general method for numerically calculating spectral norms of linear operators is to employ the \textbf{power method}
\citep{yoshida2017, miyato2018, virmaux2018, tsuzuku2018, farnia2018, gouk2018, behrmann2019invertible}, where the spectral norm of a linear operator is approximated by an iterative procedure. Here, for a randomly initialized vector $v_0$ from the domain of the convolutional operator $A$, the iteration
\begin{equation}
    v_{i+1} = A^TAv_i / \|v_i\|_2 \label{eq:power_method}
\end{equation}
converges to the leading right-singular vector of $A$ (almost surely), while $\sigma_i:=\sqrt{\|v_i\|_2}$ converges to the desired spectral norm $\|A\|_2$. The iteration (\ref{eq:power_method}) highlights the computational burden of performing the power method during training, as the most computationally demanding operation of a convolutional neural network has to be performed multiple times per gradient step until some desired precision is reached. The training time is thus effectively multiplied by (almost) the number of iterations that have to be performed. This desired precision level can be controlled by a user-specified parameter $\varepsilon$, which defines the stopping criterion $\|A^TAv_{i}- \sigma^2_{i} v_{i}\| < \varepsilon.$ The necessary number of iterations can be greatly reduced by not initializing $v_0$ randomly, but instead initializing with the estimated leading right-singular vector from the previous gradient descent step (henceforth called \textbf{warm-start power method}), assuming the most recent optimization step did not change the leading right-singular vector too much. This, however, \emph{requires keeping those vectors in memory, which are of the same size as the input features}. This can be \emph{prohibitive} in the case of memory-constrained applications such as 3D medical segmentation (where batch sizes of 1 or 2 are common).

\section{Proposed solution}\label{sec:prop_solution}
	\begin{figure*}[t!] 
			\centering
			\includegraphics[width=0.80\linewidth]{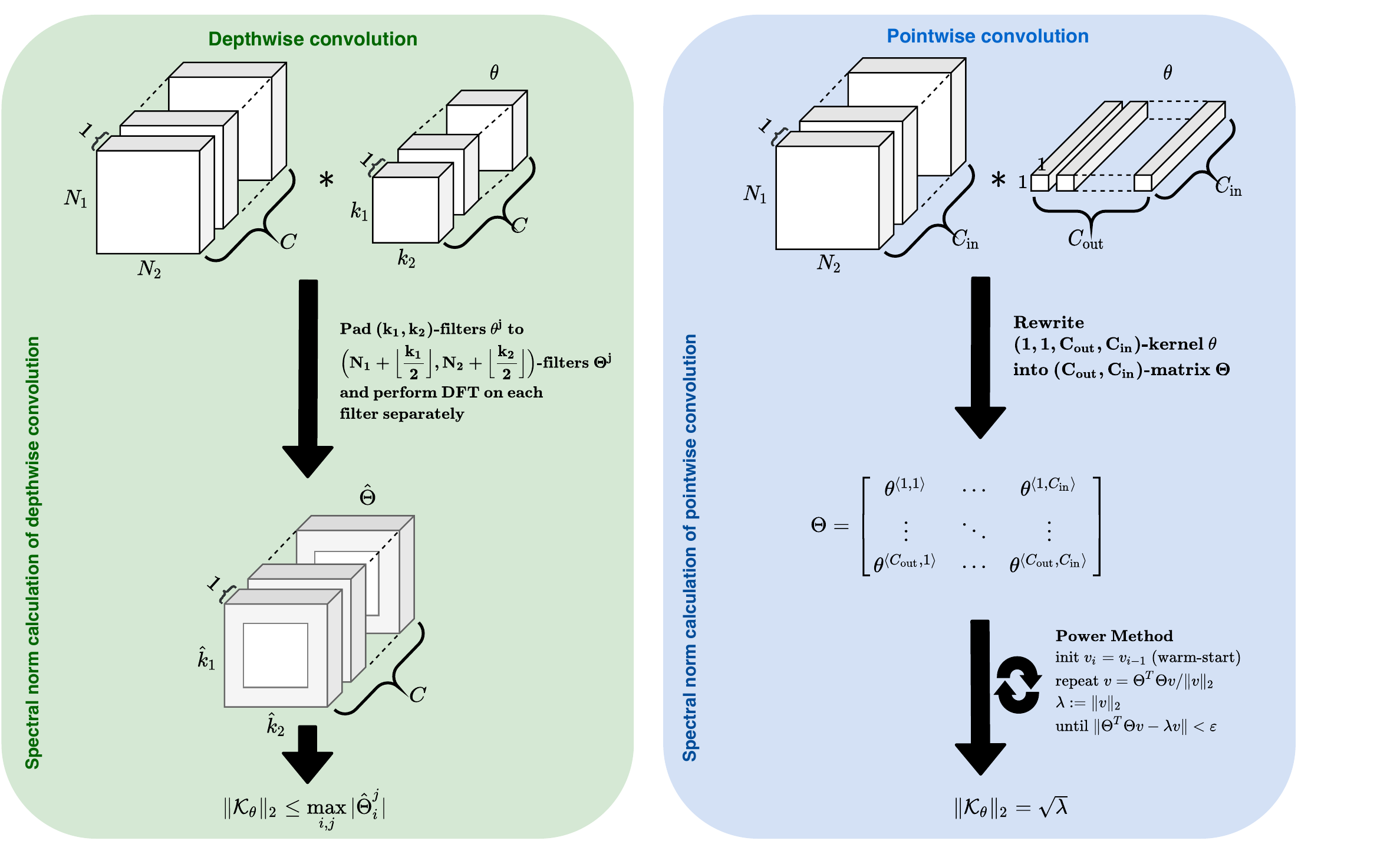}
			\caption{Our proposed approach for spectral normalization of depthwise separable convolutions (in 2D).}
			\label{fig:concept}
	\end{figure*}
	
The main difficulty in determining the spectral norm for multi-channel convolutions lies in the fact that the spatial convolutions and the interplay between different channels are inherently linked. In the case of depthwise separable convolutions however, the spatial and inter-channel computations are \emph{decoupled}. This now also allows for a separate calculation (and thus normalization, cf. section \ref{subsec:spectral_normalization}) of the individual operations' norms, which -- as it turns out -- is much faster, as well as more memory-efficient than the warm-start power method.

\subsection{Spectral norms of depthwise convolutions}\label{subsec:spec_norms_dw_convs}
Depthwise convolutions - the spatial component of depthwise separable convolutions - apply one (usually zero-padded) filtering to each input channel. While common methods for \emph{conventional} multi-channel convolutions are either slow (e.g. a randomly initialized power method or the method from \citep{sedghi2019}), produce non-tight upper bounds \citep{singla2021fantastic, araujo2020} or require a lot of additional memory (warm-start power method), the following method is accurate, fast and can be computed on-the-fly -- but is only valid for depthwise convolutions.

For this, we will make use of a form of the well-known \emph{convolution theorem}\footnote{Since most frameworks actually implement cross-correlation, we derive our formulas using the corresponding \emph{Cross-Correlation Theorem}. The final statement is, however, independent of whether we use cross-correlation or (actual) convolutions.}, which states that taking the discrete Fourier transform (DFT) of a circulant convolution between two signals is equivalent to the product of both signals' DFTs (up to a complex conjugation), such that we only need to search for the largest absolute value of the filters' Fourier coefficients. Furthermore, our derivation will be independent of the dimensionality of the data, i.e. whether we are dealing with 1D, 2D or 3D convolutions. Note that the following theorem has significant overlap with Theorem 5 from \citep{sedghi2019}, but we derive this independently from the dimensionality of the data and provide additional context, why the Fourier transform leads to the correct spectral norm, even in the case of a real-to-real convolution.

\begin{restatable}[Norms of Circulant Convolutions, Single-Channel]{theorem}{NormTheorem} \label{thm:NormTheorem}
    Let $\X=\mathbb{R}^{N_1 \times \cdots \times N_d}$ and $\mathcal{P}=\mathbb{R}^{k_1 \times \cdots \times k_d}$.
    Let $\ConvOp_\Theta: \X \to \X$ be a (single-channel, unit-stride) circulant convolution with filter $\theta \in P$.  Let $\pad: \mathcal{P} \to \X$ denote a zero-padding operator. Then $\|\ConvOp_\theta\|_2 = \max_i  | \hat{\Theta}_i|$, where $\hat{\Theta}=\DFT (\pad(\theta))$ and $\DFT$ denotes the discrete $d$-dimensional discrete Fourier transform. 
\end{restatable}
\begin{proof}
In the following, we will derive an expression for the spectral norm of the complex cross-correlation (Appendix \ref{app:CrossCorrelation}), after which we will show that this is an upper bound for the spectral norm of the real cross-correlation. Let $\Theta=\pad( \theta) \in \mathcal{X}$. Let further $\tilde{\mathcal{X}}=\mathbb{C}^{N_1 \times \cdots \times N_d}$. Let $\tilde{\ConvOp}_\theta: \tilde{X} \to \tilde{X}$ denote the complex cross-correlation, such that the restriction of $\tilde{\ConvOp}_\theta$ to $\X$ describes the same mapping as $\ConvOp_\theta$.
Then, for all $x \in \tilde{\X}$
\begin{equation}
    \begin{aligned}
	    \mathcal{F} \tilde{\ConvOp}_{\theta} x &= \overline{\DFT \Theta} \odot \DFT x,
    \end{aligned}
    \label{eq:proof_fft_spectrum_is_circ_conv_spectrum}
\end{equation}
according to the Cross-Correlation Theorem (Appendix \ref{app:CrossCorrelationTheorem}). Rewriting this in matrix-vector formulation\footnote{For a rigorous derivation, cf. Appendix \ref{app:matrix}.} yields
\begin{equation}
    \begin{aligned}
	    F \MatrixConvOp_\theta \toVec (x)= \diag(\toVec (\overline{\DFT \Theta})) \cdot F\cdot \toVec (x),
	\end{aligned}
\end{equation}
where $\toVec: \tilde{\X} \to \mathbb{C}^{N_1\cdots N_d}$ is a reordering into a column vector and $F,\MatrixConvOp_\theta \in \mathbb{C}^{(N_1\cdots N_d)\times (N_1\cdots N_d)}$ correspond to $\DFT$ and $\tilde{\ConvOp}_\theta$, respectively.
Since this holds for arbitrary $x \in \tilde{\X}$ (and $\toVec$ is orthogonal), this generalizes to $F \MatrixConvOp_\theta = \diag(\toVec(\overline{\DFT \Theta})) \cdot F,$ such that $$\MatrixConvOp_\theta = F^{-1} \cdot \diag (\toVec(\overline{\DFT \Theta})) \cdot F,$$ i.e. the entries of $\toVec(\overline{\DFT \Theta})$ constitute the eigenvalues of $\MatrixConvOp_\theta$. Because $F^{-1}=1/(N_1 \cdots N_d)\cdot F^H$ (where $F^H$ is the Hermitian transpose of $F$), it holds that $\MatrixConvOp^H_\theta \MatrixConvOp_\theta = \MatrixConvOp_\theta \MatrixConvOp_\theta^H$, i.e. $\MatrixConvOp_\theta$ is a normal matrix. According to the spectral theorem, this implies that the singular values of $\MatrixConvOp_\theta$ (and, by isomorphism, of $\tilde{\ConvOp}_\theta$) are the absolute values of the entries $\overline{\DFT \Theta}$, which coincide with the absolute values of the entries of $\hat{\Theta}:=\DFT \Theta$. 

Now, it remains to show that the spectral norms of the real-to-real operator $\ConvOp_\theta$ and the complex-to-complex operator $\tilde{\ConvOp}_\theta$ are the same. As \emph{both} are isometrically isomorphic to  $K_\theta$ (which has only real entries due to the use of a real filter $\theta$), its spectral norm is the same, whether we view it as a real matrix or as a complex matrix (see Appendix Lemma \ref{appdx:matrix_norm_equality}). Hence,
\begin{equation}
    \begin{aligned}
        &\|\ConvOp_\theta\|_2 = \max_{\substack{x\in \mathbb{R}^n \\ \|x\|_2=1}} \| \MatrixConvOp_\theta x\|_2  = \max_{\substack{x\in \mathbb{C}^n \\ \|x\|_2=1}} \| \MatrixConvOp_\theta x\|_2 \\
        =& \|\tilde{\ConvOp}_\theta\|_2 = \max_i | \hat{\Theta}_i|,
    \end{aligned}
\end{equation}
finishing our proof.

\end{proof}

While the above theorem yields exact spectral norms, it has two shortcomings, as it only deals with \emph{circulant} and \emph{single-channel} convolutions. In practice however, depthwise convolutions are typically zero-padded (and multi-channel). As it turns out, by zero-padding the filters to just a bit bigger than the input image (to the size of the padded image), an in practice quite tight upper bound of the actual spectral norm of the depthwise, zero-padded convolution is found.

\begin{restatable}[Norms of Zero-Padded and Multi-Channel Depthwise Convolutions]{theorem}{MultiNormTheorem} 
\label{thm:multi_norm_theorem}
Let $\X=\mathbb{R}^{N_1 \times \cdots \times N_d}$ and $\X^C=\mathbb{R}^{C \times N_1 \times \cdots \times N_d}$.
For odd-dimensional filters $\theta^1, \dots, \theta^C \in \mathcal{P}= \mathbb{R}^{(2p_1+1)\times \cdots \times (2p_d+1)}$ (for $p_i\in \mathbb{N}_0$), let $\theta = (\theta^1, \dots, \theta^C)$. We define
\begin{equation}
    \begin{aligned}
        \ConvOp_\theta: X^C \to & X^C \\
         x^{\langle j \rangle} \mapsto & \kappa_{\theta^j} x^{\langle j \rangle} \text{ for each channel } j
    \end{aligned}
\end{equation}
where $\kappa_{\theta^j}: \X \to \X$ denotes the single-channel zero-padded convolution with filter $\theta^j$. Let $$\pad: \mathcal{P} \to \mathbb{R}^{(N_1+2p_1)\times \cdots \times (N_d+2p_d)}$$ denote the zero-padding. Then, with $\hat{\Theta}^j:=\DFT (\pad(\theta^j))$, it holds that $$\|\ConvOp_\theta\|_2 \leq \max_{i,j}  | \hat{\Theta}^j_i|,$$ where $\DFT$ denotes the discrete $d$-dimensional discrete Fourier transform. 
\end{restatable}
The proof is found in Appendix \ref{appdx:depthwise_proof}. Our proposed method based on this Theorem is found in Figure \ref{fig:concept}.

Since there are very fast and highly parallelizable algorithms for computing the DFT (the \emph{Fast Fourier Transform}), the above computation can be performed very quickly. 

Note that the above theorems are derived for unit-stride depthwise convolutions. For strided convolutions, the upper bound still holds, but it will in reality be more vacuous. To see this, let $\mathcal{S}$ denote the strided convolution and let $\ConvOp$ denote the unit-strided convolution. Then there is a subsampling operator $P$, such that $\mathcal{S}=P \ConvOp$. Due to $\|P\|_2 = 1$, it holds that $\|\mathcal{S}\|\leq \|P\|_2 \|\ConvOp\|_2 = \|\ConvOp\|_2 $.  In the case of stride $\nu >1$ in all directions, one can only make the 'educated guess' based on the assumption that all components of $\ConvOp x$ contribute to its overall norm evenly, in which case $\nu^{d}\|P\ConvOp x\|^2_2 \approx \|\ConvOp x\|^2_2$. This e.g. means that for 2D data, our estimation is about 2 times too high for $\nu=2$. 
\subsection{Calculating Lipschitz constants for pointwise convolutions}\label{subsec:lip_const_pw_convs}
As discussed in section \ref{subsec:dw_conv}, pointwise convolutions are simply multi-channel convolutions
$$y^{\langle i \rangle}=\textstyle{\sum_j} \theta^{\langle i, j \rangle} \ast x^{\langle j \rangle}$$
with $\theta^{\langle i, j \rangle} \in \mathbb{R}$, which is why in 2D they are often called $(1 \times 1)$-convolutions. Here, we propose a variation of the warm-start power method for a specific matrix, which -- compared to the na\"{i}ve power method for the full operator -- is typically much faster and memory-efficient. In particular, it is independent of the size of the convolved data.

\begin{theorem}
Let $\ConvOp_\theta$ be a pointwise convolution with kernel $\theta$, such that $\theta^{\langle i,j \rangle} \in \mathbb{R}$ for all $i \in \{1,\dots,C_\text{out}\}$ and $j \in \{1,\dots,C_\text{in}\}$, where $C_\text{in}$ and $C_\text{out}$ denote the number of input respectively output channels. Then we call 
\begin{equation}
    \Theta :=     
    \begin{bmatrix}
        \theta^{\langle 1,1 \rangle} & \cdots & \theta^{\langle 1,C_\text{in} \rangle} \\
        \vdots & \ddots & \vdots \\
        \theta^{\langle C_\text{out},1 \rangle} & \cdots & \theta^{\langle C_\text{out}, C_\text{in} \rangle} 
    \end{bmatrix} \in \mathbb{R}^{C_\text{out} \times C_\text{in}},
\end{equation}
the \emph{connectivity matrix} of $\ConvOp_\theta$ and it holds that $$\|\ConvOp_\theta\|_2 = \|\Theta\|_2.$$
\end{theorem}
\begin{proof}
$\ConvOp_\theta$ is similar to a block diagonal matrix with blocks $\Theta$. As the singular values of block diagonal matrices are the union of the blocks' singular values, the spectral norm (i.e. the largest singular value) of $\ConvOp_\theta$ coincides with the spectral norm of $\Theta$.
\end{proof}

Performing the power method with the connectivity matrix $\Theta$ has the additional advantage, that the warm-start version of the power method requires only storing a small right-singular vector\footnote{Alternatively, the left-singular vector from $\mathbb{R}^{C_\text{out}}$ can be stored.} from $\mathbb{R}^{C_\text{in}}$. Compare that to the na\"{i}ve power method, which requires storing a vector of the size of the input (respectively output) tensor of the layer. We refer to the thus-improved warm-start power method as the \textbf{efficient warm-start power method}. An illustration of this method is found in Figure \ref{fig:concept}.

For a square images of resolution $N$-by-$N$, both the computational and memory demand is divided by a factor of $N^2$. This effect is even more pronounced in case of 3D (or even higher-order) data, where the memory cost for the na\"{i}ve method can easily be several gigabytes, whereas our method is typically limited to kilobytes of storage for storing the leading singular vector (for at most a few thousand channels).

\section{Experiments}\label{sec:experiments}
In this section, we perform several experiments on classical benchmarks in order to evaluate our proposed method for spectral normalization. Note that the goal here is not to convince the reader of the usefulness of spectral normalization in general (interested readers are referred to the many examples in sections \ref{sec:introduction} and \ref{subsec:specnorm_for_conv_layers}), but to benchmark the accuracy and speed of our method and to uncover the interplay with other variables, such as the learning rate.

\subsection{Accuracy of depthwise spectral norm upper bound}\label{subsec:acc_dw_conv_bound}
While the accuracy of our efficient power method for \emph{pointwise} convolutions is governed by the user-specified $\varepsilon$-parameter, for depthwise convolutions we only have access to upper bounds of the true spectral norm. As shown in Theorem \ref{thm:multi_norm_theorem}, the upper bound comes into play solely because of zero-padding, whose \emph{relative} effect should become smaller for large resolutions of the input images (in comparison to small kernel sizes). In order to check how tight this upper bound is (and thus how close our approximation is), we calculate the relative overestimation (approximation/actual value) for randomly generated (Gaussian distributed) $3 \times 3$ filters, depending on the feature resolution. Here, the 'true' spectral norm was approximated with 30 iterations of the power method. We varied the feature resolution from $7 \times 7$ (the smallest resolution in e.g. MobileNetV2  \cite{sandler2018mobilenetv2}, VGG \cite{simonyan15vgg} and ResNet \cite{he2016deep}) to 128 and generated 1000 random filters per resolution. The results are shown in Figure \ref{fig:random_kernel_acc}. One can observe that the relative error is indeed decreasing with an increase in the resolution, from (in median) 17\% to 2\%.
Note that due to the linearity of the investigated methods, the results of this experiment are independent of the standard deviation of the chosen Gaussian distribution.
 
The above analysis can be seen as an investigation into the approximation accuracy at initialization. So how good is the approximation in a fully-trained network, where the filters do not follow a Gaussian distribution? For this, we did the above comparison for a MobileNetV2, which was pretrained on ImageNet \cite{russakovsky2015imagenet}. MobileNetV2 contains 17 depthwise convolutional layers, 13 of which are unit-stride and 4 of which have a stride of 2 in each direction. The correct feature resolution from the original article was used for the calculation of the spectral norms. For the 13 unit-stride convolutions, the average relative overestimation was 2.80\%, with a standard deviation of 1.55\%. In the case of the strided convolutions (for which we expect more vacuous bounds, as explained in section \ref{subsec:spec_norms_dw_convs}), the relative overestimation of our method was 75.6\%
(standard deviation: 20.6\%). In summary, in a realistic setting, our estimation of the spectral norm is highly accurate for unit-stride depthwise convolutions, but less accurate for strided convolutions (but still better than the 'educated guess' of a factor of 2, which we laid out in section \ref{subsec:spec_norms_dw_convs}). 
\begin{figure}[h]
    \centering
    \includegraphics[width=.45\textwidth]{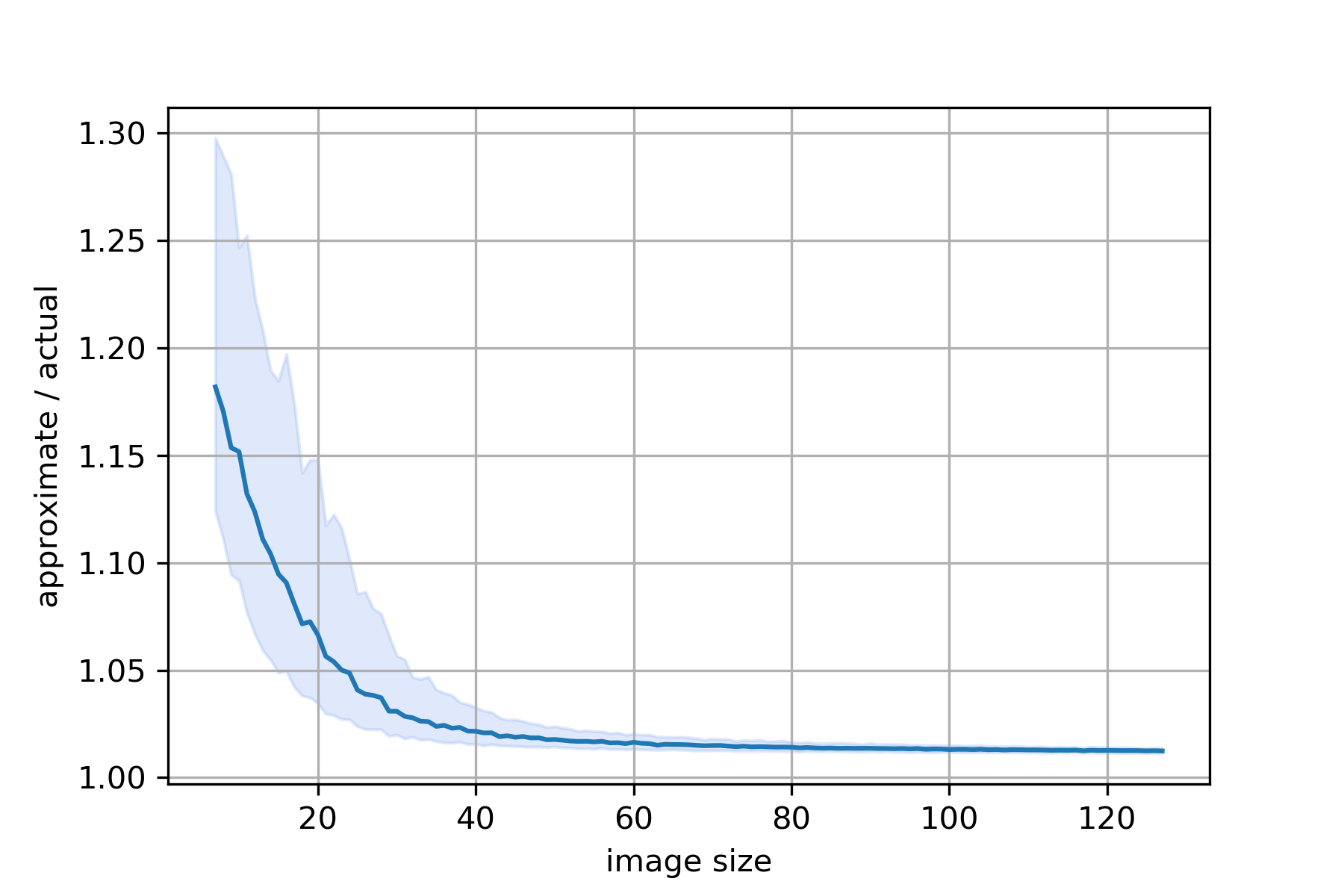}
    \caption{Relative overestimation of our DFT-based spectral norm upper bound for standard Gaussian kernels and varying image size. The dark blue line denotes the median over 1000 random kernels, while the shaded region shows the interquartile range.}
    \label{fig:random_kernel_acc}
\end{figure}

\subsection{Lipschitz constant and learning rate}\label{subsec:lip_const_lr}
Since the Lipschitz constant of a neural network determines the magnitude of its gradients, training with gradient-based methods requires adapting the learning rate accordingly -- i.e. a lower Lipschitz constant necessitates a larger learning rate.
In order to evaluate the interplay of the learning rate and scaling constant $K$ (both for soft and hard scaling), we perform a study on CIFAR-10 \cite{krizhevsky2009learning}, in which we test different settings for both variables. For this, we train a ResNet-34 \citep{he2016deep}, for which we exchanged all convolutions by depthwise separable convolutions (one depthwise convolution, followed by a pointwise convolution without intermediate activation function). We therefore end up with a network architecture consisting of 33 depthwise separable convolutional layers (of which 3 are strided convolutions with a stride of 2) and one fully connected (softmax) layer. We trained all networks for 300 epochs, using scaling constants from 1 to 10 and learning rates $10^{-5}, 10^{-4}, 10^{-3}, 10^{-2} \text{ and } 10^{-1}$ for stochastic gradient descent with a fixed momentum of $0.9$. Since batch normalization influences the Lipschitz constant (both through the running means of the batches' standard deviations and the $\gamma$-parameter), we train our models without batch normalization.
We compute and restrict the spectral norm of every layer, using the proposed method to spectrally normalize the depthwise separable convolutions, as well as the warm-start power method to restrict the spectral norm of the final classification layer.  Additionally, we choose the precision parameter $\varepsilon$ to be 0.01 for all pointwise convolutional layers and initialize the soft scaling parameter as $s=3$ so that $K \cdot \tanh(s)\approx K$ at initialization. 

The results of this experiment are depicted in Table \ref{tab:lr_ablation_study}. As predicted, lower scaling constants $K$ tend to require higher learning rates. Furthermore, the known regularizing effect of the Lipschitz constraints can be observed, since increasing $K$ initially increases the prediction accuracy (underfitting regime), before it decreases again (overfitting regime). A careful tuning of the learning rate is needed to achieve the best performance. In particular, too high of a learning rate may collapse the training. Moreover, there is no clear winner between hard and soft scaling.

\begin{table*}
	\centering
	\caption{Best learning rate per scaling constant $K$: Test accuracy (in percent) for a ResNet-34 with depthwise separable convolutions trained on the CIFAR-10 dataset. Accuracies for hard scaling (H) and soft scaling (S) are compared for learning rates from 1e-5 to 1e-1. It can be seen that  the optimal learning rate decreases with increasing scaling constant.}
	\begin{tabular}{c c c c c c c c c c c}
		\toprule
		& \multicolumn{10}{c}{\textbf{Learning rate}}\\ %
		& \multicolumn{2}{c}{\textbf{1e-5}} & \multicolumn{2}{c}{\textbf{1e-4}} & \multicolumn{2}{c}{\textbf{1e-3}} & \multicolumn{2}{c}{\textbf{1e-2}} & \multicolumn{2}{c}{\textbf{1e-1}} \\ %
		\textbf{Scaling constant} $\mathbf{K}$ & H & S & H & S & H & S & H & S & H & S \\ \midrule%
		1 &\ 15.38 &\  14.98 &\ \textbf{16.64} &\ 15.56 &\ 16.34 &\  16.28 &\ 10.00 &\ 10.00  &\ 10.00 &\ 10.00 \\  
		2 &\ 17.23 &\ 17.05 &\  19.93 &\ 19.00 &\ 45.01 &\ 42.17 &\ 54.70 &\ \textbf{55.65} &\ 10.00 &\ 10.00 \\  
		3 &\ 18.02 &\ 17.48 &\  45.88 &\ 42.08 &\ 67.37  &\ 72.51 &\ \textbf{74.73} &\ 72.33 &\ 10.00 &\ 10.00 \\  %
		4 &\ 35.45 &\ 35.34 &\ 62.90 &\ 62.73 &\ 81.07 &\ 80.94 &\ 84.71 &\ \textbf{85.65} &\ 10.00 &\ 10.00 \\  %
		5 &\ 50.68 &\ 50.36 &\ 74.85 &\ 74.73 &\ 86.68 &\ 86.72 &\ 90.53 &\ \textbf{90.74} &\ 10.00 &\ 10.00 \\  %
		6 &\ 59.73 &\ 59.65 &\ 80.45 &\ 80.19 &\  \textbf{88.60} &\ 87.63 &\ 10.00 &\ 10.00 &\ 10.00 &\ 10.00 \\  %
		7 &\ 62.10 &\ 62.39 &\ 81.10 &\ 81.06 &\ 85.24 &\ \textbf{85.36} &\ 10.00 &\ 10.00 &\  10.00 &\ 10.00 \\  %
		8 &\ 61.53 &\ 62.74 &\ \textbf{81.21} &\ 81.17 &\ 80.16 &\ 10.00 &\  10.00 &\ 10.00 &\  10.00 &\ 10.00 \\  %
		9 &\ 45.52 &\ \textbf{52.25} &\ 32.69 &\ 41.92 &\ 10.00 &\ 10.00 &\ 10.00 &\ 10.00 &\ 10.00  &\ 10.00 \\  %
		10 &\ 10.00 &\ 10.00 &\ 10.00 &\ 10.00 &\ 10.00 &\ 10.00 &\ 10.00  &\ 10.00 &\ 10.00  &\ 10.00 \\ \bottomrule%
	\end{tabular}
	\label{tab:lr_ablation_study}
\end{table*}
    
\subsection{Spectral normalization on ImageNet}\label{subsec:spec_norm_imagenet}
In addition to the previous CIFAR-10 experiments, we conducted a study on the much higher-resolution ImageNet dataset \citep{russakovsky2015imagenet}. For this, we trained MobileNetV2 \citep{sandler2018mobilenetv2}, a standard architecture for image classification using depthwise separable convolutions, using $K \in \{10, 12, 15, 20, 30, 40\}$, and again without batch normalization. Lower and higher values for $K$ in our case lead a strong breakdown in performance. The models were trained with soft scaling for 150 epochs, using the precision parameter $\varepsilon = 0.01$ for every pointwise convolutional layer. All models were trained with a batch size of 128, divided on two NVIDIA V100 GPUs. Due to the interplay of scaling constant and learning rate (see subsection \ref{subsec:lip_const_lr}), the learning rate had to be adapted for each $K$.

In Figure \ref{fig:imagenet_acc} the Top-1 as well as the Top-5 validation accuracy per scaling constant $K$ are plotted. One can observe that the accuracy increases as one increases the scaling constant, indicating that we never reach the overfitting regime.
Moreover, it can be seen that the increase in accuracy is steeper in the lower regions of $K$, later leveling off.

The use of soft scaling offers an additional degree of flexibility, as the training can slightly adjust the Lipschitz constant. We noticed that in the ImageNet experiments, the initialization of the learnable parameter $s$ had to be tuned. For lower values of $K$ (up to $K=20$), the last initialization from the previous CIFAR-10 experiments ($s=3$, resulting in $\tanh(3) \approx 0.995$) still worked very well, whereas higher values of $K$ required reducing the initial value of $s$ somewhat (to  $s=0.5$ so that $\tanh(s) \approx 0.46$), otherwise resulting in exploding gradients or missing convergence. We also for instance experimented with $s=0$ or $s=\arctanh (1/K)$ (which initializes the network as if there was no spectral normalization at initialization), but both did not generally work well. However, in all cases it was enough to monitor the training for a few gradient steps to judge whether an adjustment of the initial value for $s$ was necessary.

\begin{figure}[h]
    \centering
    \includegraphics[width=\linewidth]{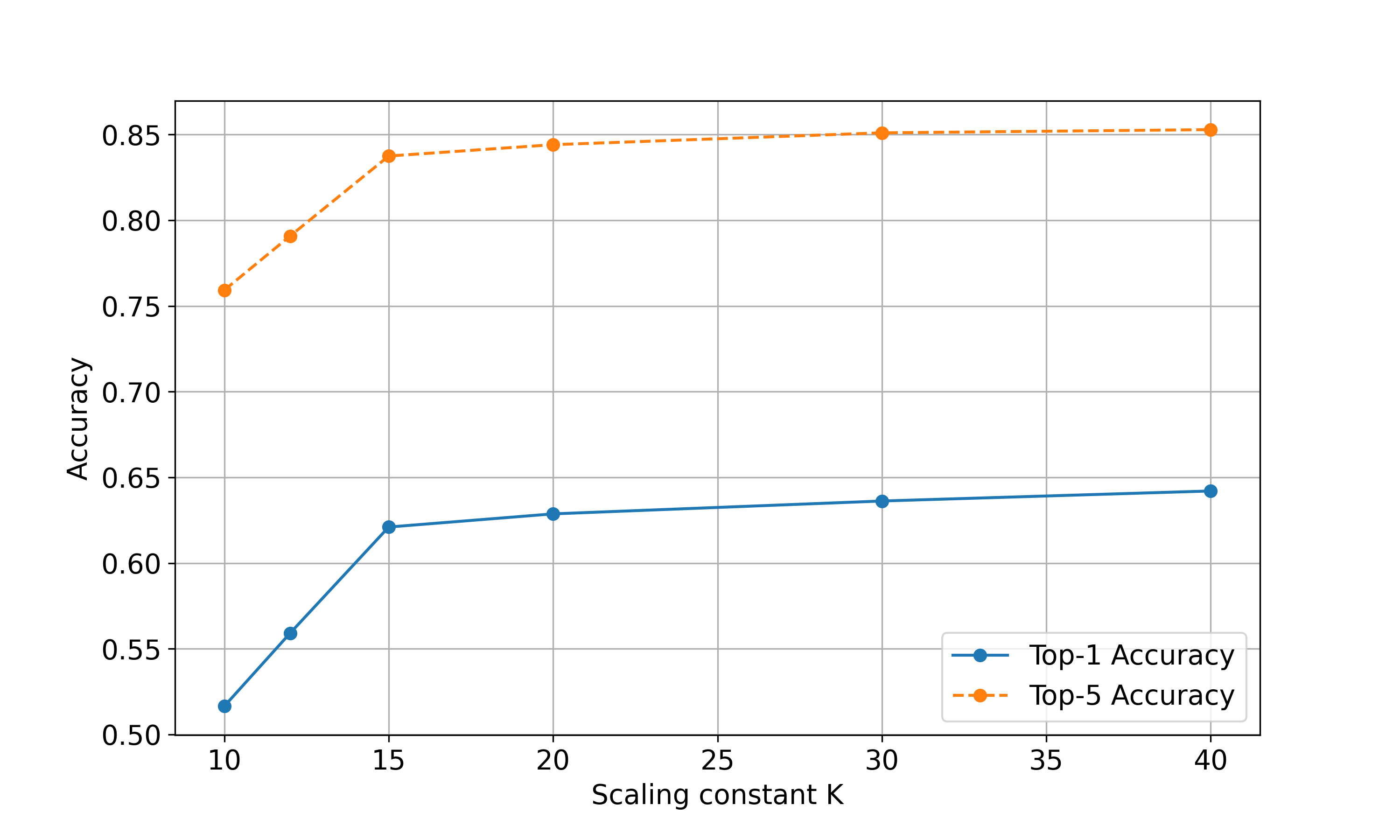}
    \caption{ImageNet validation accuracy for a MobileNetV2 with spectral normalized layers for scaling constants from 10 to 40. Expectedly, loosening the upper bound per layer leads to increasing accuracies.}
    \label{fig:imagenet_acc}
\end{figure}

\subsection{Time complexity}\label{subsec:time_complexity}
In the following, we aim to show the impact of our method on training times, which we measured both with and without spectral normalization. Instead of benchmarking only linear layers, this includes the time for data loading, augmentation, backpropagation etc, which allows for a realistic assessment of the method's overhead in a real training scenario. In addition to the absolute times per epoch, we show the relative factor between our approach and training the network without spectral normalization.

For this, we observe the average training times per epoch of the last two experimental cohorts (ResNet-34 on CIFAR-10 and MobileNetV2 on ImageNet, best accuracy model each), and compare these to the case of the same network \emph{without spectral normalization} (and still without batch normalization). The results are summarized in Table \ref{tab:time_per_epoch}. 

In the case of CIFAR-10, our additional overhead amounted to an increase of about 63\% (where a considerable amount was spent on the normalization of the final, fully-connected layer). In the case of ImageNet, the time per epoch only increased by 2\%, which can be seen as negligible. Since the cost of our normalization is small in comparison to convolutions on large features, it is to be expected that our method has less of an impact on the training time for higher-dimensional data.

\begin{table}[h]
    \vspace{15pt}
	\centering
	\caption{Time per epoch for CIFAR-10 classification with a ResNet-34 with depthwise separable convolutions and ImageNet classification with a MobileNetV2 architecture.}
	\begin{tabular}{c c c c}
	\toprule
	    & \multicolumn{2}{c}{\textbf{Time per epoch}} & \\%
	 \textbf{} & \textbf{SpecNorm} & \textbf{no SpecNorm} & \textbf{Relative} \\\midrule%
	 CIFAR-10 & 59.59\ s & 36.63\ s & 1.63\\%
	 ImageNet & 3074\ s & 3016\ s & 1.02\\%
	\end{tabular} \newline
	\label{tab:time_per_epoch}
\end{table}

\section{Conclusion \& future work}\label{sec:conclusion}
Mathematical guarantees in the form of Lipschitz constants have increasingly come into the focus of modern neural networks research. While there are simple and generic ways of calculating (and thus enforcing) spectral norms (resulting in said Lipschitz guarantees), these are typically quite costly in time or memory, or result in inaccurate upper bounds. For application for which efficiency is of interest, depthwise separable convolutions have emerged as a fast form of convolutions, which nonetheless allow for state-of-the-art results in challenging benchmarks (cf. \citet{pham2020meta}). In this work, introduce a very simple procedure, with which the spectral normalization of depthwise separable convolutions can be performed with negligible computational and memory overhead, while being quite accurate in practice.

We imagine future work based on our proposed method both from an application as well as a methodological perspective. From a methodological standpoint, going from the spectral norm to different norms is another possible avenue. In the case of the pointwise convolution, this is straightfoward, as this just entails computing matrix norms (which are often fast to compute, e.g. in the case of the 1-norm or $\infty$-norm). 
In terms of applications, research into architectures that make full use of depthwise separable convolutions for any of the applications named in section \ref{sec:introduction} is needed (e.g. spectrally normalized GANs), which can then benefit from our proposed method. Making the step to 3D will yield even bigger performance improvements compared to the conventional approaches. One use case is the stabilization of invertible networks for memory-efficient 3D segmentation \citep{etmann2020iunets}, which can be achieved by limiting the layer's Lipschitz constant \citep{behrmann2020}, e.g. via our method. 
\section*{Acknowledgements}
CE acknowledges support from the Wellcome Innovator Award RG98755.
CBS acknowledges support from the Philip Leverhulme Prize, the Royal Society Wolfson Fellowship, the EPSRC grants EP/S026045/1 and EP/T003553/1, EP/N014588/1, EP/T017961/1, the Wellcome Innovator Award RG98755, European Union Horizon 2020 research and innovation programme under the Marie Skodowska-Curie grant agreement No. 777826 NoMADS, the Cantab Capital Institute for the Mathematics of Information and the Alan Turing Institute.
\bibliography{literature}
\bibliographystyle{icml2021}
\onecolumn
\appendix
{\Large {Appendix}}
\section{Additional Details and Proofs} 
In the following, we will give additional details and proofs for the theory in Section \ref{sec:prop_solution}. For compactness, we will use a multi-index notation indicated by bold letters and numbers. All multi-indices have $d$ entries, meant to represent $d$-dimensional data (e.g. $d=2$ for images), i.e. $\mathbf{n}=(n_1,\dots,n_d)$. We will also use the shorthand $\mathbb{R}^\mathbf{N}:=\mathbb{R}^{N_1 \times \cdots \times N_d}$, respectively $\mathbb{C}^\mathbf{N}:=\mathbb{C}^{N_1 \times \cdots \times N_d}$. For some $f \in \mathbb{R}^\mathbf{N}$ (respectively $f \in \mathbb{C}^\mathbf{N}$), we write $f[\mathbf{n}]:=f[n_1,\dots,n_d]:=f_{n_1,\dots,n_d}$. Division and addition are meant entry-wise. Furthermore, we define the sum notations
$$\sum_{\mathbf{n=0}}^{\mathbf{N-1}} f[\mathbf{n}] := \sum_{n_1=0}^{N_1-1} \cdots \sum_{n_d=0}^{N_d-1} f[n_1, \dots, n_d]$$
and
$\langle \mathbf{a}, \mathbf{b} \rangle := a_1b_1 + \cdots + a_d b_d$. We denote the imaginary unit by $i$, as well as the complex conjugate of a complex vector $f$ by $\overline{f}$.

\subsection{Discrete Fourier Transform} \label{app:DFT}
Recall the following definition of the (unnormalized) discrete Fourier transform (DFT). 
\begin{definition} We call
\begin{equation}
    \begin{aligned}
        \DFT : \mathbb{C}^\mathbf{N} \to& \phantom{.} \mathbb{C}^\mathbf{N} \\
         f \mapsto& \phantom{.} \hat{f},
    \end{aligned}
\end{equation}
defined by
\begin{equation}
    \begin{aligned}
        \hat{f}[\mathbf{j}] := \sum_{\mathbf{n=0}}^{\mathbf{N-1}} f[\mathbf{n}] e^{-i 2\pi \mathbf{ \langle n/N , j \rangle}}
    \end{aligned}
\end{equation}
the discrete Fourier transform.
\end{definition}
Compared to the 'normalized' version of the discrete Fourier transform (which is multiplied by $1/\sqrt{N_1 \cdots N_d}$ and is unitary), this unnormalized formulation of the DFT results in slightly simpler statements for the main results.

\subsection{Circulant Cross-Correlation} \label{app:CrossCorrelation}
Here, we will define the complex version of the cross-correlation operation. Note that virtually all neural network libraries implement convolutions as cross-correlation.
\begin{definition}
For $\Theta,f \in \mathbb{C}^\mathbf{N}$, we call the vector $g \in \mathbb{C}^\mathbf{N}$ defined by
\begin{equation}
    \begin{aligned}
        g[\mathbf{n}] = \sum_{\mathbf{k=0}}^{\mathbf{N-1}} \overline{\Theta[\mathbf{n}]} f[\mathbf{n+k}]
    \end{aligned}
\end{equation}
the \emph{circulant cross-correlation of $\Theta$ with $f$}, if $f$ is circularly continued, i.e. $f[\dots, n_j + mN_j, \dots]=f[\dots, n_j, \dots]$ for all $n_j \in \{0,\dots,N_j-1\}$ for $j \in \{1,\dots,d\}$ and $m\in \mathbb{Z}$. We denote this cross-correlation by $\Theta \ast f:=g$.
\end{definition}

\subsection{Circulant Cross-Correlation Theorem} \label{app:CrossCorrelationTheorem}
\begin{theorem} Let $g = \Theta \ast f$ and $\hat{g}=\DFT g$. Then 
$$\hat{g}[\mathbf{j}] = \overline{(\mathcal{F} \Theta)[\mathbf{j}] } \cdot (\mathcal{F}f)[\mathbf{j}],$$
or, written differently, $$\DFT (\Theta \ast f) = \overline{(\DFT \Theta)} \odot (\DFT f),$$
where $\odot$ denotes the component-wise multiplication.
\end{theorem}

\begin{proof}
\allowdisplaybreaks
    \begin{align*}
    \hat{g}[\mathbf{j}] = & \sum_{\mathbf{n=0}}^{\mathbf{N-1}} g[\mathbf{n}] e^{-i 2\pi \langle \mathbf{n/N, j} \rangle} \\
        = & \sum_{\mathbf{n=0}}^{\mathbf{N-1}}  \sum_{\mathbf{k=0}}^{\mathbf{N-1}} \overline{\Theta[\mathbf{k}]} f[\mathbf{n+k}] e^{-i 2\pi \langle \mathbf{n/N, j} \rangle} \\ 
        = & \sum_{\mathbf{n=0}}^{\mathbf{N-1}}  \sum_{\mathbf{l=0}}^{\mathbf{N-1}} \overline{\Theta[\mathbf{k}]} f[\mathbf{n+k}] e^{-i 2\pi \langle \mathbf{(n+k-k)/N, j} \rangle} \\
        = & \sum_{\mathbf{n=0}}^{\mathbf{N-1}}  \sum_{\mathbf{l=0}}^{\mathbf{N-1}} \overline{\Theta[\mathbf{k}]} f[\mathbf{n+k}] e^{-i 2\pi \langle \mathbf{(n+k)/N, j} \rangle} e^{-i 2\pi \langle \mathbf{-k/N, j} \rangle} \\
        = & \sum_{\mathbf{k=0}}^{\mathbf{N-1}} \overline{\Theta[\mathbf{k}]}  e^{-i 2\pi \langle \mathbf{-k/N, j} \rangle} \sum_{\mathbf{n=0}}^{\mathbf{N-1}} f[\mathbf{n+k}] e^{-i 2\pi \langle \mathbf{(n+k)/N, j} \rangle} 
    \end{align*}

Using the fact that $\exp (-ix)=\overline{\exp (ix)}$ for $x\in \mathbb{R}$:
\begin{equation*}
    \begin{aligned}
        \sum_{\mathbf{k=0}}^{\mathbf{N-1}} \overline{\Theta[\mathbf{k}]}  e^{-i 2\pi \langle \mathbf{-k/N, j} \rangle} = \sum_{\mathbf{k=0}}^{\mathbf{N-1}} \overline{\Theta[\mathbf{k}] e^{-i 2\pi \langle \mathbf{k/N, j} \rangle} } = \overline{\textstyle{ \sum_{\mathbf{k=0}}^{\mathbf{N-1}}} \Theta[\mathbf{k}] e^{-i 2\pi \langle \mathbf{k/N, j} \rangle} } = \overline{(\mathcal{F} \Theta)[\mathbf{j}]} 
    \end{aligned}
\end{equation*}

Due to the circular continuation of $f$ and the fact that the exponential function is $2\pi$-periodic on the imaginary axis, we can simplify the rightmost sum to the Fourier transform of $f$:
\begin{equation*}
    \begin{aligned}
        \sum_{\mathbf{n=0}}^{\mathbf{N-1}} f[\mathbf{n+k}] e^{-i 2\pi \langle \mathbf{(n+k)/N, j} \rangle} = (\mathcal{F}f)[\mathbf{j}]
    \end{aligned}
\end{equation*}
In summary, this proves the statement $$\hat{g}[\mathbf{j}] = \overline{(\mathcal{F} \Theta)[\mathbf{j}] } \cdot (\mathcal{F}f)[\mathbf{j}].$$
\end{proof}

\subsection{Matrix Formulation of the Cross-Correlation Theorem} \label{app:matrix} Here, we make the reformulation of the cross-correlation theorem into matrix multiplications (which are used to derive the statement in Theorem \ref{thm:NormTheorem}) more precise.
\begin{remark} \normalfont
Let $\tilde{\X}= \mathbb{C}^{N_1 \times \cdots \times N_d}$, $\vec{\X}= \mathbb{C}^{N_1 \cdots N_d}$. Let further $V:=\toVec: \tilde{\X} \to \vec{\X}$ be an operation that reorders a vector from $\tilde{X}$ into a (column) vector in $\vec{X}$, which is a unitary operator (i.e. $V^{-1}=V^\ast$, where $V^\ast$ is the adjoint of $V$). Representing an operator in a different coordinate system entails first transforming to the required domain of the operation (here: via the transition operator $V^\ast$), and then transforming the output to the desired codomain (here: via $V$). This means that for the linear operators $\DFT$ and $\ConvOp_\theta$ to be reformulated into matrices, we define $F:=V\DFT V^\ast$ and $\MatrixConvOp_\theta:= V \tilde{\ConvOp}_\theta V^\ast$. Likewise, the entry-wise multiplication in $\tilde{X}$ (formulated as a bilinear operator)
\begin{equation}
    \begin{aligned}
    M: \tilde{\X} \times \tilde{\X} &\to \tilde{\X} \\
             (a,b)&\mapsto a \odot b
    \end{aligned}
\end{equation}
is equivalently represented as an operation in $\vec{X}$ via
\begin{equation}
    \begin{aligned}
    \vec{M}: \vec{\X} \times \vec{\X} &\to \vec{\X} \\
             (c,d)&\mapsto V\cdot M(V^\ast c ,V^\ast d).
    \end{aligned}
\end{equation}
According to the Cross-Correlation Theorem \ref{app:CrossCorrelationTheorem}, it holds that $\DFT \ConvOp_\theta x = M( \DFT \Theta, \DFT x)$ for all $x \in \tilde{\X}$. Due to the unitarity of $V$, it holds that $\DFT = V^\ast F V$ and $\ConvOp_\theta=V^\ast K_\theta V$, meaning that
$$\DFT \ConvOp_\theta x = (V^\ast F V) (V^\ast K_\theta V) x = V^\ast F K_\theta \toVec{(x)}$$
and
$$M( \DFT \Theta, \DFT x) =  M( (V^\ast F V) \Theta, (V^\ast F V) x) = M( V^\ast \toVec (\hat{\Theta}), V^\ast F \toVec (x)).$$ Left-multiplying by $V$ then leads to
$$F \MatrixConvOp_\theta \toVec (x) = \vec{M} ( F \toVec{(\Theta)},  F \toVec{(x)} ),$$
due to the unitarity of $V$. Note that entry-wise multiplication in $\vec{\X}$ can be represented by the multiplication of a diagonal matrix with a column vector, which results in the statement $$F \MatrixConvOp_\theta \toVec (x) = \diag (F \toVec{(\Theta)}) \toVec{(x)}.$$
\end{remark}

\subsection{Matrix Spectral Norm Equality} \label{appdx:matrix_norm_equality}
The spectral norm of the real-to-real convolution (according to Theorem \ref{thm:NormTheorem}) is derived using its eigenvalues, when viewed as a complex-to-complex operator. Since eigenvalues of operations defined on spaces over $\mathbb{R}$ and $\mathbb{C}$ differ, we now assert that the spectral norm of matrix with only real entries (and by extension, our real-to-real convolutional operator) does not depend on whether we view it as a matrix in $\mathbb{R}^{n \times n}$ or as a matrix in $\mathbb{C}^{n \times n}$.
\begin{lemma} \label{appdx:matrix_norm_equality}
We denote the spectral norm for real matrices $K \in \mathbb{R}^{n \times n} $ by $$\|K\|_{2,\mathbb{R}^n}:=\max_{\substack{x\in \mathbb{R}^n \\ \|x\|_2=1}} \|Kx \|_2,$$ as well as the spectral norm for complex matrices $\tilde{K} \in \mathbb{C}^{n \times n}$ by $$\|\tilde{K}\|_{2,\mathbb{C}^n}:=\max_{\substack{x\in \mathbb{C}^n \\ \|x\|_2=1}} \|\tilde{K}x \|_2.$$
Then for any $K \in \mathbb{R}^{n \times n}$, it holds that $$\|K\|_{2,\mathbb{R}^n} = \|K\|_{2,\mathbb{C}^n}.$$
\end{lemma}
\begin{proof}
We will prove the above statement by showing that both $\|K\|_{2,\mathbb{R}^n} \leq \|K\|_{2,\mathbb{C}^n}$ and $\|K\|_{2,\mathbb{R}^n} \geq \|K\|_{2,\mathbb{C}^n}$.\\

It's easy to see that
$$\|A\|_{2,\mathbb{R}^n}=\max_{\substack{x\in \mathbb{R}^n \\ \|x\|_2=1}} \|Ax \|_2 \leq \max_{\substack{x\in \mathbb{C}^n \\ \|x\|_2=1}} \|Ax \|_2 = \|A\|_{2,\mathbb{C}^n}, $$
since $\mathbb{R}^n \subset \mathbb{C}^n$.

On the other hand, since $K$ only has real entries, $\MatrixConvOp^T \MatrixConvOp$ is symmetric and real, which means that there is an orthogonal matrix $S$ which diagonalizes $\MatrixConvOp^T \MatrixConvOp$, i.e. $$\MatrixConvOp^T \MatrixConvOp= S^T \overbrace{\diag (\sigma^2_1, \dots, \sigma^2_n)}^{=:D} S,$$ where $\sigma_1 \geq \dots \geq \sigma_n$ are the singular values of $\MatrixConvOp$.

Then, when viewed as a real matrix, 
$$\|\MatrixConvOp\|^2_{2,\mathbb{R}^n}=\max_{\substack{x\in \mathbb{R}^n \\ \|x\|_2=1}} \| \MatrixConvOp x\|^2_2 =\max_{\substack{x\in \mathbb{R}^n \\ \|x\|_2=1}} x^T \MatrixConvOp^T \MatrixConvOp x= \max_{\substack{x\in \mathbb{R}^n \\ \|x\|_2=1}} (Sx)^T D (S x)=\max_{\substack{y\in \mathbb{R}^n \\ \|y\|_2=1}} y^T D y = \max_{\substack{y\in \mathbb{R}^n \\ \|y\|_2=1}} \sum_{k=1}^n \sigma^2_k y^2_k = \sigma_1^2.$$
Note that $S$, as an orthogonal matrix, is in particularly a unitary matrix, i.e. $S^H=S^T=S^{-1}$. Thus,
\begin{align*}
    &\|\MatrixConvOp\|^2_{2,\mathbb{C}^n}=\max_{\substack{x\in \mathbb{C}^n \\ \|x\|_2=1}} \| \MatrixConvOp x\|^2_2 =\max_{\substack{x\in \mathbb{C}^n \\ \|x\|_2=1}} x^H \MatrixConvOp^H \MatrixConvOp x= \max_{\substack{x\in \mathbb{C}^n \\ \|x\|_2=1}} (Sx)^H D (S x)=\max_{\substack{y\in \mathbb{C}^n \\ \|y\|_2=1}} y^H D y = \max_{\substack{y\in \mathbb{C}^n \\ \|y\|_2=1}} \sum_{k=1}^n \sigma^2_k |y_k|^2 \\
    \leq &\max_{\substack{y\in \mathbb{C}^n \\ \|y\|_2=1}} \sum_{k=1}^n \sigma^2_1 |y_k|^2 = \sigma_1^2 = \|\MatrixConvOp\|^2_{2,\mathbb{R}^n}.
\end{align*}

\end{proof}

\subsection{Zero-padded multi-channel convolutions} \label{appdx:depthwise_proof} 
A single-channel zero-padded convolution's spectral norm is upper bounded by the spectral norm of the circulant convolution on an enlarged domain.
\begin{lemma}\label{lem:PaddingLemma}
Let $\theta$ be a filter of odd spatial dimensions, i.e. $\theta \in \mathcal{P}:=\mathbb{R}^\mathbf{k}$ with $\mathbf{k}=\mathbf{2p+1}$ for $p_i \in \mathbb{N}_0$ for all $i \in \{1,\cdots,d\}$. Let
$$\ConvOp^\text{zero}_{\theta,\mathbf{{N}}}: \mathbb{R}^\mathbf{N} \to  \mathbb{R}^\mathbf{N}$$
denote the zero-padded single-channel convolution with filter $\theta$ on the domain $\mathbb{R}^ \mathbf{N}$ and let for $\overline{\mathbf{N}}=\mathbf{N+2p}$
$$\ConvOp^\text{circ}_{\theta,\mathbf{\bar{N}}}:\mathbb{R}^\mathbf{N+2p} \to \mathbb{R}^\mathbf{N+2p}$$
denote the respective circulant convolution on the domain $\mathbb{R}^\mathbf{N+2p}$. Then 
$$\|\ConvOp^{\textrm{ \normalfont zero}}_{\theta,\mathbf{{N}}}\|_2 \leq \|\ConvOp^{\textrm{ \normalfont circ}}_{\theta,\mathbf{\bar{N}}}\|_2.$$ 
\end{lemma}
\begin{proof}
Let $\pad^\text{zero}_{\mathbf{M} \to \mathbf{M'}}$ and $\pad^\text{circ}_{\mathbf{M} \to \mathbf{M'}}$ be the zero-padding respectively circulant padding operation from $\mathbb{R}^\textbf{M}$ to $\mathbb{R}^\textbf{M'}$ (for some $\mathbf{M}$ and $\mathbf{M}'$). Both can be easily verified to be linear. We define $\overline{\mathbf{N}}=\mathbf{N+2p}$ and $\overline{\overline{\mathbf{N}}}=\mathbf{N+4p}$.

Let ${\ConvOp^{\text{valid}}_{\theta,\mathbf{M}}} : \mathbb{R}^{\mathbf{M}} \to \mathbb{R}^{\mathbf{M-2p}}$ be the valid convolution with filter $\theta$ on the domain $\mathbb{R}^\textbf{M}$. 

Note that for any $x \in \X$, it holds that
\begin{equation}
    \begin{aligned}
        \| \ConvOp^{\text{zero}}_{\theta,\mathbf{N}} \cdot x \|_2 = \|\ConvOp^{\text{valid}}_{\theta,\overline{\mathbf{N}}} \cdot \pad^\text{zero}_{\mathbf{N} \to \overline{\mathbf{N}}} \cdot x\|_2
        = \| \ConvOp^{\text{valid}}_{\theta,\overline{\overline{\mathbf{N}}}} \cdot \pad^\text{circ}_{\overline{\mathbf{N}} \to \overline{{\overline{\mathbf{N}}}}} \cdot \pad^\text{zero}_{\mathbf{N} \to \overline{\mathbf{N}}} \cdot x\|_2, 
    \end{aligned}
\end{equation}
since at every equation, only zeros are added to the sum-of-squares when calculating the 2-norms.
It follows that
\begin{equation}
    \begin{aligned}
        \| \ConvOp^{\text{zero}}_{\theta,\mathbf{N}} \|_2 =& 
        \max_{\substack{x\in \X \\ \|x\|_2=1}} \| \ConvOp^{\text{valid}}_{\theta,\overline{\overline{\mathbf{N}}}} \cdot \pad^\text{circ}_{\overline{\mathbf{N}} \to \overline{{\overline{\mathbf{N}}}}} \cdot \pad^\text{zero}_{\mathbf{N} \to \overline{\mathbf{N}}} \cdot x \|_2 \\
        \leq &  \| \underbrace{\ConvOp^{\text{valid}}_{\theta,\overline{\overline{\mathbf{N}}}} \cdot \pad^\text{circ}_{\overline{\mathbf{N}} \to \overline{{\overline{\mathbf{N}}}}}}_{=\ConvOp^{\text{circ}}_{\theta,\mathbf{\overline{N}}}} \|_2 \cdot \underbrace{\| \pad^\text{zero}_{\mathbf{N} \to \overline{\mathbf{N}}}\|_2}_{=1} \\
        = & \| \ConvOp^{\text{circ}}_{\theta,\mathbf{\overline{N}}} \|_2,
    \end{aligned}
\end{equation}
finishing our proof.
\end{proof}

\MultiNormTheorem*
\begin{proof}
The zero-padded multi-channel depthwise convolution is isometrically isomorphic to a block diagonal matrix $M$, where each block $M_k$ is isometrically isomorphic to a matrix representing the zero-padded single-channel convolution. With Lemma \ref{lem:PaddingLemma} and the fact that the set of singular values of $M$ is the union of all singular values of the $M_k$, the statement follows immediately.
\end{proof}

\newpage

\section{Experimental details}
This section details the hyperparameters, settings and data preprocessing steps used for the experiments in section \ref{sec:experiments}.

\subsection{Lipschitz constant and learning rate}\label{subsec:app_lip_const_lr_exp_details}
An overview of the training settings for the CIFAR-10 Lipschitz constant and learning rate experiment can be found in Table \ref{tab:overview_training_settings_lr_experiment}. We trained the networks on a NVIDIA V100 GPU with scaling constants from 1 to 10 and learning rates $10^{-5}, 10^{-4}, 10^{-3}, 10^{-2}, 10^{-1}$ both with soft and hard scaling.
\begin{table}[h]
    \centering
    \caption{Overview of training settings for CIFAR-10 learning rate experiment.}
    \begin{tabular}{l r}
    \toprule
      \textbf{Parameter}   &  \textbf{CIFAR-10} \\ \midrule
      Network architecture & ResNet-34 \\ 
      Layers & Replaced conventional convs by depthwise separable convs \\
      Dataset & CIFAR-10  \\ 
      Loss function & Cross Entropy Loss \\ 
      Optimizer & SGD (Momentum: 0.9)\\ 
      Learning rate & $10^{-5}$, $10^{-4}$, $10^{-3}$, $10^{-2}$, $10^{-1}$\\ 
      Learning rate schedule & Multiplication of learning rate with $0.1$ at epoch milestones\\ 
      Scheduler milestones & 150, 250 \\ 
      Epochs & 300 \\ 
      Batch size & 128\\ 
      Spectral normalization & True \\ 
      Scaling constant & 1, 2, 3, 4, 5, 6, 7, 8, 9, 10 \\ 
      Pointwise convolution $\varepsilon$ & 0.01 \\
      Soft/Hard scaling & Soft / Hard \\ 
      Initialization learnable parameter $s$ & 3 / - \\ \bottomrule
    \end{tabular}
    \label{tab:overview_training_settings_lr_experiment}
\end{table}

We use the CIFAR-10 dataset \cite{krizhevsky2009learning} with the default train/test split and the following data pre-processing steps during training and testing:
\begin{enumerate}
    \item[a.] \textbf{Training}: 
        \begin{itemize}
            \item random cropping (with a padding of size 4) to images of size $32 \times 32$
            \item random horizontal flipping
            \item normalizing per channel with mean $[0.4914, 0.4822, 0.4465]$ and standard deviation $[0.2023, 0.1994, 0.2010]$
        \end{itemize}
    \item[b.] \textbf{Testing}:
    \begin{itemize}
    \item normalizing per channel with mean $[0.4914, 0.4822, 0.4465]$ and standard deviation $[0.2023, 0.1994, 0.2010]$
    \end{itemize}
\end{enumerate}

\subsection{Spectral normalization on ImageNet}\label{subsec:app_spec_norm_exp_details}
An overview of the hyperparameters used for the ImageNet classification experiments is shown in Table \ref{tab:overview_imagenet_exp}. We trained the networks on two NVIDIA V100 GPUs with scaling constants 10, 12, 15, 20, 30, 40 and learning rates $10^{-3}$, $10^{-3}$, $10^{-3}$, $5\cdot 10^{-4}$, $5\cdot 10^{-4}$, $3\cdot 10^{-4}$ respectively. 
\begin{table}[h]
    \centering
    \caption{Overview of training settings for ImageNet classification experiments.}
    \begin{tabular}{l r}
    \toprule
      \textbf{Parameter}   &  \textbf{ImageNet} \\ \midrule
      Network architecture & MobileNetV2 \\ 
      Layers & as in \citep{sandler2018mobilenetv2} \\
      Dataset & ImageNet \\ 
      Loss function & Cross Entropy Loss\\ 
      Optimizer & SGD (Momentum: 0.9)\\ 
      Learning rate & $10^{-3}$/ $10^{-3}$/ $10^{-3}$/ $5\cdot 10^{-4}$/ $5\cdot 10^{-4}$/ $3\cdot 10^{-4}$\\ 
      Learning rate schedule & Multiplication of learning rate with $0.1$ at epoch milestones\\ 
      Scheduler milestones & 50, 100 \\ 
      Epochs & 150 \\ 
      Batch size & 128\\ 
      Spectral normalization & True\\ 
      Scaling constant & 10/ 12/ 15/ 20/ 30/ 40 \\ 
      Pointwise convolution $\varepsilon$ & 0.01 \\
      Soft/Hard scaling & Soft\\ 
      Initialization learnable parameter $s$ & 3/ 3/ 3/ 3/ 0.5/ 0.5 \\\bottomrule
    \end{tabular}
    \label{tab:overview_imagenet_exp}
\end{table}

We use the ImageNet dataset \cite{russakovsky2015imagenet} with the default train/val/test split and the following data pre-processing during training and validation:
\begin{enumerate}
    \item[a.] \textbf{Training}: 
        \begin{itemize}
            \item cropping of random size and resizing to images of size $224 \times 224$
            \item random horizontal flipping
            \item normalizing per channel with mean $[0.485, 0.456, 0.406]$ and standard deviation $[0.229, 0.224, 0.225]$
        \end{itemize}
    \item[b.] \textbf{Validation}:
    \begin{itemize}
    \item resizing image to size $256 \times 256$
    \item center cropping to size $224 \times 224$
    \item normalizing per channel with mean $[0.485, 0.456, 0.406]$ and standard deviation $[0.229, 0.224, 0.225]$
    \end{itemize}
\end{enumerate}

\subsection{Time complexity}
For the time complexity experiments, we trained identical networks with hyperparameters shown in Table \ref{tab:overview_cifar_timing} and Table \ref{tab:overview_imagenet_timing} -- with and without spectral normalization for CIFAR-10 and ImageNet classification. The data train/test split as well as the pre-processing steps for CIFAR-10 and ImageNet can be found in subsection \ref{subsec:app_lip_const_lr_exp_details} and \ref{subsec:app_spec_norm_exp_details}, respectively.
\begin{table}[h]
    \centering
    \caption{Overview of training settings for CIFAR-10 timing experiment.}
    \begin{tabular}{l r r}
    \toprule
      \textbf{Parameter}   &  \textbf{CIFAR-10 w SpecNorm} & \textbf{CIFAR-10 w/o SpecNorm} \\ \midrule
      Network architecture & ResNet-34 & ResNet-34\\ 
      Layers & Standard with DS Convs & Standard with DS Convs\\
      Dataset & CIFAR-10 & CIFAR-10\\ 
      Loss function & Cross Entropy Loss & Cross Entropy Loss \\ 
      Optimizer & SGD (Momentum: 0.9) & SGD (Momentum: 0.9)\\ 
      Learning rate & $10^{-2}$ & $10^{-2}$\\ 
      Learning rate schedule & Multiply lr with $0.1$ at milestones & Multiply lr with $0.1$ at milestones\\ 
      Scheduler milestones & 150, 250 & 150, 250\\ 
      Epochs & 300 & 300\\ 
      Batch size & 128 & 128\\ 
      Spectral normalization & True & False \\ 
      Scaling constant & 5 & - \\ 
      Pointwise convolution $\varepsilon$ & 0.01 & -\\
      Soft/Hard scaling & Hard & -\\ 
      Initialization learnable parameter $s$ & - & - \\
      Activations depthwise separable convolutions & None & None \\ \bottomrule
    \end{tabular}
    \label{tab:overview_cifar_timing}
\end{table}
\begin{table}[h]
    \centering
    \caption{Overview of training settings for ImageNet timing experiments.}
    \begin{tabular}{l r r}
    \toprule
      \textbf{Parameter}   &  \textbf{ImageNet w SpecNorm} & \textbf{ImageNet w/o SpecNorm}\\ \midrule
      Network architecture & MobileNetV2 & MobileNetV2 \\ 
      Layers & as in \citep{sandler2018mobilenetv2}  & as in \citep{sandler2018mobilenetv2}\\
      Dataset & ImageNet & ImageNet\\ 
      Loss function & Cross Entropy Loss & Cross Entropy Loss \\ 
      Optimizer & SGD (Momentum: 0.9) & SGD (Momentum: 0.9)\\ 
      Learning rate & $3 \cdot 10^{-4}$ & $10^{-3}$\\ 
      Learning rate schedule & Multiply lr with $0.1$ at milestones & -\\ 
      Scheduler milestones & 50, 100 & - \\ 
      Epochs & 150 & 5\\ 
      Batch size & 128 & 128\\ 
      Spectral normalization & True & False\\ 
      Scaling constant & 40 & -\\ 
      Pointwise convolution $\varepsilon$ & 0.01 & -\\
      Soft/Hard scaling & Soft & -\\ 
      Initialization learnable parameter $s$ & 0.5 & - \\
      Activations depthwise separable convolutions & ReLU6 & ReLU6\\ \bottomrule
    \end{tabular}
    \label{tab:overview_imagenet_timing}
\end{table}

\end{document}